\documentclass[submission,copyright,creativecommons]{eptcs}
\providecommand{\event}{TARK 2019} % Name of the event you are submitting to
\usepackage{breakurl}             % Not needed if you use pdflatex only.
\usepackage{underscore}           % Only needed if you use pdflatex.
\usepackage{amsmath}
\usepackage{amssymb}
\usepackage{amsthm}
\usepackage{graphicx}
\usepackage{color}
\usepackage{xparse}
\usepackage{hyperref}

\title{Aggregating Probabilistic Judgments}
 
 \author{Magdalena Ivanovska
 \institute{University of Oslo\\ Oslo, Norway}
 \email{magdalei@ifi.uio.no}
 \and
 Marija Slavkovik
 \institute{University of Bergen\\
 Bergen, Norway}
 \email{\quad marija.slavkovik@uib.no}
 }
\def\titlerunning{Aggregating Probabilistic Judgments}
\def\authorrunning{M. Ivanovska \& M. Slavkovik}
\newtheorem{example}{Example}[section]
\newtheorem{definition}{Definition}[section]
\newtheorem{theorem}{Theorem}[section]
\newtheorem{proposition}{Proposition}[section]

\newcommand{\La }{\mbox{$\mathcal{L}$} }       

\newcommand{\A}{\Phi}

\newcommand{\Ct}{\Gamma}  

\newcommand{\Cp}{\hat{\Gamma}}

\newcommand{\Js}{J}

\newcommand{\Jp}{\hat{\Js}}

\newcommand{\ai}{\varphi}

\newcommand{\Pc}{P}

\newcommand{\Ph}{\hat{P}}

\newcommand{\acj}{\mathcal{J}}

\newcommand{\acp}{\mathbb{P}}

\newcommand{\app}{\hat{\mathbb{P}}}

\newcommand{\qv}{\mathbf{c}}

\newcommand{\q}{c}

\newcommand{\paf}{(\Phi, N, \Ct, \Cp)}

\newcommand{\caf}{(\Phi, N, \Ct)}

\newcommand{\qu}{q}

\newcommand{\F}{F}
\newcommand{\Fh}{\hat{F}}

\newcommand{\Agg}{F}

\newcommand{\Fq}{\hat{F}_S}

\newcommand{\avg}[1]{{E_{\Ph}(#1)}}
  \newcommand{\ie}{\mbox{\textit{i.e.}, }}
 \newcommand{\eg}{\mbox{\textit{e.g.}, }}
 \newcommand{\Eg}{\mbox{\textit{E.g.}, }}

\newcommand \argmax[1] {\underset{{#1}}{\mbox{argmax}}}
\newcommand \argmin[1] {\underset{{#1}}{\mbox{argmin}}}

\begin{document}

\maketitle

\begin{abstract}
In this paper we explore the application of methods for classical judgment aggregation in pooling probabilistic opinions on logically related issues. For this reason, we first modify the Boolean judgment aggregation framework in the way that allows handling probabilistic judgments and then define probabilistic aggregation functions obtained by generalization of the classical ones. In addition, we discuss essential desirable properties for the aggregation functions and explore impossibility results. 
\end{abstract}

\section{Introduction}
Judgment aggregation (JA) is concerned with aggregating sets of binary truth valuations assigned to logically related issues \cite{ListPuppe2009,GrossiP:2014}. Various collective decision making problems  in artificial intelligence can be modelled as JA problems,  \eg problems of constructing agreements, such as  finding a collective goal in multi-agent systems \cite{Synthese12new,COIN10}. In agreement reaching problems each agent in a group  is a source of judgments and also typically affected by the collective choice resulting from the aggregation of individual judgments. %\MScomment{COMMENT P1C1L7: 
For example,  I am a citizen voting on a referendum that decided not to impose global warming curbing methods, but I am also a citizen that has to live with the consequences of that collective decision.
A typical JA example \cite{ListPuppe2009} is one concerning  three issues:  Current $CO_2$ emissions lead to global warming ($p$),  If current $CO_2$ emissions lead to global warming, then we should
reduce $CO_2$ emissions ($p\rightarrow q$), We should reduce $CO_2$ emissions ($q$).  The individual sets of judgments are as in Table~\ref{tab:co2example}. As observed from the example, pooling the truth valuations on each issue does not always lead to a consistent set of collective judgments. JA designs and studies aggregators that produce a consistent outcome.

\begin{table}
\centering
\begin{tabular}{r|ccc}
& $p$ & $p \rightarrow q$ & $q$ \\ \hline
Minister 1   & true & true                          & true\\
Minister 2   & true & false                         & false\\
Minister 3   & false & true                         & false\\\hline
Majority     & true & true                           & false 
\end{tabular}
%\medskip
\caption{An example of a judgment aggregation.} \label{tab:co2example} 
\vspace{-0.2cm}
\end{table}

However, aggregation problems are not always Boolean, because the individual judgments on whether an issue is true or false are not always certain. We give an example.

\begin{example}
\label{ex:running} 
You want a recommendation for a specific hotel, ``The Grand Palace'', however, you want that recommendation to be  compiled specifically for you. You are interested in:  
\begin{itemize}
\item a hotel  close to the centre or well connected with public transport ($s \vee t$);
\item a hotel that  is a unique experience ($x$), 
\item a hotel that is a good value for money ($a$).  
\end{itemize}
  The information that you can get  from online information sources (IS), like booking.com, TripAdvisor, etc.,  can be processed automatically, by pooling reviews from recommendations regarding  ``The Grand Palace'' hotel. An example of such collection of opinions is given in Table~\ref{tab:hotel}.  What we obtain from each IS is the likelihood that an issue is true.    You can find information online about  $s \vee t$ (second column), about whether ``The Grand Palace'' hotel is a unique experience  (third column) and also whether the hotel is recommended by the users, $h$, (fourth column). However, it is not enough that the hotel is recommended in the reviews. For you,  a hotel should be recommended ($h$) iff   both $s \vee t$ and $a$ are true, i.e. $((s \vee t) \wedge a) \leftrightarrow h$. Information about $a$ may not be available to extract. Assume that you define that a hotel is a good value for money ($a$) if it is not more than 80 Euro per night ($\neg e$) or  if it is a unique experience, \ie  when $(\neg e \vee  x) \leftrightarrow a$ holds.   Then the information you need to extract is whether $\neg e$ is true. This is given in the sixth column of Table~\ref{tab:hotel}.
  
  \begin{table}[h!]
\centering 
%\vspace{0.5cm}
\begin{tabular}{r| ccccc}
                   &   $s \vee t$& $x$   & $h$ & $a$   & $\neg e$\\ \hline %& $b \wedge p$ & $e$ & $h$
IS 1   		&   0.6              &  1       & 1   & - &1 \\ % &0.6& 0.3 & 0.3 \\
IS 2  		 &     0.7          &  0.6    & 0.5 & - & 1 \\ %&0.6&0.5 &0.7\\
IS 3  		 &     0.1          &  0.4    & 0.2 & - & 1\\ %&0.3& 1 & 0.5\\ 
IS 4  		 &     0.8          &  0.8   & 0.9 & - &1 \\ % &0.6& 0.3 & 0.3 \\
IS 5  		 &     0.7          &  0.7   & 0.4  & - &1\\ %&0.6&0.5 &0.7\\
IS 6   		&    0.5              &  0.6   & 0.3 & - & 1   %&0.3& 1 & 0.5\\ 
\end{tabular}
%\medskip
\caption{An example of a source aggregation.} \label{tab:hotel}  
\vspace{-0.2cm}
\end{table}
%
%
%
% by analysing and merging information on hotels from online sources. 
%Your requirements are: 
%
% \begin{itemize}
%\item a hotel  close to the centre or well connected with public transport ($s \vee t$);
%\item a hotel that  is a unique experience ($x$),
%\item a hotel that is a good value for money ($a$).  
%
%
%\item a hotel is recommended  ($h$) iff    both $s \vee t$ and $a$ are true, i.e. $((s \vee t) \wedge a) \leftrightarrow h$
%\item a hotel is a good value for money ($a$) if it is not more than 80 Euro per night ($\neg e$) or  if it is a unique experience, \ie  when $(\neg e \vee  x) \rightarrow a$ holds.  
%\end{itemize}
%Each online information source (IS)   Thus for example the reviews from {\em Booking.com} will yield one truth valuation on the requirements,   {\em TripAdvisor} another etc. However, the information sources will yield  probabilistic judgments on the requirements  rather than a Boolean truth value assignment. Table \ref{tab:hotel} gives an example of six such pooled information sources. What we obtain from each IS is the likelihood that an issue is true.   
\end{example} 
 
 We want to be able to aggregate likelihood judgments like the ones represented in the rows of Table~\ref{tab:hotel}, but  into a set of Boolean judgments: should the hotel be recommended and for which reasons.  To achieve this purpose, we explore how methods from classical judgment aggregation can be adjusted to deal with probabilistic statements as judgments. Thus,  we  extend the propositional logic JA framework typically used \cite{ListPuppe2009,GrossiP:2014,LangPSTV15,EndrissGHL16} using  the logic of likelihood \cite{Halpern03}, and design probabilistic aggregation functions based on the classical ones.  Thus, intuitively, what were desirable properties for aggregation in the classical case, remain desirable properties in the probabilistic framework.
 
   Our framework  allows sources to have uncertain probabilistic judgments that are {\em rational} and subject to inevitable probabilistic constraints,  but also to aggregate into a  collective judgment set that is   Boolean  and   subject to a specific set of propositional constraints. 
 Frameworks for representing non-binary judgments have been considered, see \eg \cite{GrossiP:2014} for an overview, however no specific methods for aggregation have been designed for these frameworks.  Rather, impossibility characterisations have been studied showing which sets of desirable properties cannot be mutually satisfied. Here we propose specific classes of aggregators for the framework we introduce.

There is a certain amount of literature on probabilistic opinion pooling (\eg see \cite{MartiniS2017} for a detailed survey) which is concerned with aggregating probability functions (representing opinions of agents) into a single one. The defined properties of the aggregating functions are similar to those in JA theory, and similarly as there, impossibility results are proved.  However, opinion pooling presumes that every agent has its probabilistic judgments defined on a $\sigma$-algebra of events (or, equivalently on a set of possible worlds). Despite the inherent consistency, this is not always a realistic requirement. 
%Namely, we cannot expect that the information sources  provide the probabilities of all the conjunctions of literals of the variables in order to determine the probability function on the full $\sigma$-algebra determined  by these variables. Neither can we expect these probabilities to be determined by the likelihood information like the one provided in Table \ref{tab:hotel}. 
In our framework, we allow the agents to express their probabilistic opinions on any (logically related) propositional language statements (equivalently, on any subset of a $\sigma$-algebra of events),  and, moreover, these opinions can be imprecise, i.e., expressed through likelihood inequalities (equivalently, a set of probability functions is provided by each source.) In this sense our work is more comparable to variants of opinion pooling that presume a general agenda \cite{DietrichList2010} or deal with imprecise probabilities \cite{stewart2018}. Aggregation functions in probabilistic opinion pooling are typically averaging functions like (weighted) linear or geometrical average. Here we take the approach of defining aggregation functions by generalizing judgment aggregation functions based on representative voting.

The paper is structured as follows. In Section~\ref{sec:framework} we introduce the judgment aggregation framework based on the logic of likelihood. In Section~\ref{sec:Probabilistic aggregation functions} we demonstrate how to generalise classical judgment aggregation functions into functions that handle likelihood judgments, and we also introduce some new classes of judgment aggregation functions. In Section~\ref{sec:properties} we discuss desirable properties of aggregation functions and revisit the classical impossibility results. In Section~\ref{sec:related} we discuss related work and in Section~\ref{sec:conclusion}  we make our conclusions and outline directions for future work.%We first define our framework providing justification of the modelling choices. It is convenient for us to build the new framework upon the preliminaries from classical JA which are also given in this section. We then define what we consider to be desirable and adequate properties for likelihood aggregators in this new framework and state  the ``typical" impossibility results. We next focus on defining classes of aggregators by relaxing some of the requirements,  observing at the same time the connections with the corresponding ``classical'' aggregators.   Lastly we discuss related work, and indicate directions for future work.

\section{Framework}
\label{sec:framework}

We distinguish between an {\em agenda setter} and {\em information sources}. %{that may or may not be \em agents}.
The agenda setter identifies the set of issues, \ie the {\em agenda} for which Boolean judgments need to be made. The agenda setter can also set  additional relations, which we call {\em propositional constraints}, that should hold among the agenda issues. The information sources are modelled as sets of {\em likelihood formulas} subject to different relations  called {\em probabilistic constraints}. The probabilistic constraints model the natural and contextual properties of the issues. 

\subsection{Judgment aggregation model}
To model the agenda and the propositional constraints we use a set $\La$  of propositional logic formulas. 
An agenda is a finite set  $\A ~\subset~\La$, 
\begin{equation}
\label{eq:agenda}
\A=\{\varphi_1,\ldots,\varphi_m \}\;,
\end{equation}
s.t. $\varphi_i$ is neither a tautology nor a contradiction. We call the elements of the agenda  {\it issues}. The set of {\em propositional constraints} $\Ct \subset \La$  represents special relations that should hold among the agenda issues described by the agenda setter.  $\Ct$ should be satisfiable, and we allow $\Ct=\{\top\}$.  In Example~\ref{ex:running}, we have $\A= \{  s\vee t, h, x,e, a\}$, and $\Ct = \{(\neg e \vee  x) \leftrightarrow a, ( (s \vee t) \wedge a) \leftrightarrow h \}$.

The agenda setter is interested in  aggregating collections of judgments on the agenda issues from various information sources into a set of crisp (Boolean) judgments that is consistent with $\Ct$. A {\em crisp judgment} on $\ai \in \A$ is either $\ai$ or $\neg \ai$. A {\em crisp judgment set} $\Js$  is a set of crisp judgments. \Eg the judgments of Minister 2 in  Table\ref{tab:co2example} can be represented as a crisp judgment set $\Js_2=\{p, \neg(p\rightarrow q), \neg q \}$.
We introduce the notation $\A^\cup = \A \cup \{ \neg \ai \mid \ai \in \A\} $. Then a crisp judgment set $\Js$ is a subset of $\A^\cup$. % containing either $\varphi$ or $\neg \varphi$, for \MIcomment{some of the agenda issues $\varphi$}. 
The set $\Js$ is {\em consistent} if $\Js \cup \Ct $ is a consistent set of formulas in classical propositional logic. $\Js$ is {\em complete} if it contains one crisp judgment for each of the issues in the agenda. If the crisp judgment set $\Js$ is consistent and complete, we say that it is {\em rational}. 
Given an agenda $\A$ and propositional constraints $\Ct$, the set of all consistent  and complete, \ie rational crisp judgment sets is  $\acj$.
 
We model the information sources   as   {\em sets of likelihood  judgments} on $\A^{\cup}$. 
 A {\it likelihood judgment} on the issue $\varphi\in \A^{\cup}$ is a simple likelihood formula of the type:
\begin{equation}
\label{eq:likelihood formula}
\ell(\varphi)\geq^* a\;,
\end{equation}
where $\geq^*\in \{\geq,  =\}$ and $a\in [0,1]$.\footnote{The formula (\ref{eq:likelihood formula}) is an instance of the logic of likelihood in \cite{FaginHalpernMegiddo90}, \cite{Halpern03} that consists of Boolean combinations of linear likelihood formulas of the type  
$a_1\ell(\varphi_1)+\ldots+a_n\ell(\varphi_n)\geq b$,
where $a_i, b$ are real numbers, and $\varphi_i$ are pure propositional formulas.
Likelihood formulas are interpreted in probability spaces where the term $\ell(\ai)$ is interpreted as the probability of the set of worlds (outcomes) at which $\ai$ is true.
}

The likelihood judgment $\ell(\varphi)\geq a$ expresses that the likelihood (probability)\footnote{ In this paper we interpret likelihood as probability and we use the two terms interchangeably. Note that, however, likelihood can also be interpreted as other measure of belief, see \cite{Halpern03}.} of the statement $\varphi$ being true is at least $a$. This intuition immediately implies that $\ell(\neg \varphi)\leq 1-a$. This and other entailments we mention later can formally be proved in the axiomatic system for the logic of likelihood that consists of axioms for propositional reasoning, reasoning about inequalities, and the following axioms for probabilistic reasoning given in \cite{Halpern03}:\\
	 (L1)  $\ell(\varphi)\geq 0$, \\     
	 (L2)  $\ell(\top)=1$,\\     
	 (L3)  $\ell(\varphi)=\ell(\varphi\wedge\psi)+\ell(\varphi\wedge \neg \psi)$,\\
	 (L4)  From $\varphi \leftrightarrow \psi$ infer $\ell(\varphi)=\ell(\psi)$.
 
Having $\ell(\neg \varphi)\leq 1-a$ gives us an upper, but not a lower bound for the likelihood of $\neg \varphi$. Therefore, we ask that an explicit  judgment for the likelihood of $\neg \varphi$ is given. 

$\ell(\varphi) = a$ is a stronger statement than $\ell(\varphi) \geq a$ expressing that the likelihood of $\varphi$ being true is exactly $a$. In this case, a judgment for $\neg\varphi$ is implied,  namely, $\ell(\neg \varphi) = 1- a$.

Each of the  information sources is represented as 
a {\it set of likelihood judgments} $\Jp$. The set $\Jp$  has one likelihood judgment on each of the issues in  $\A^{\cup}$:
\begin{align}
\label{eq:judgment}
\Jp=&\{\ell(\varphi)\geq^* a_{\varphi} \mid \varphi\in \A^{\cup}\}\;, 
\end{align}
where $\geq^*\in \{=, \geq \}$, $a_{\varphi}\in [0,1]$.

Note that providing likelihood formulas for both $\varphi$ and $\neg\varphi$ in Eq.(\ref{eq:judgment}) is equivalent with  providing intervals for the likelihood of either $\varphi$ or $\neg\varphi$  (hence the information sources are free to do that) but for the discussion in this paper the formulation in Eq.(\ref{eq:judgment}) is a more suitable one.

A set of likelihood judgments is always  {\it complete} in the sense that  it contains a likelihood judgment for each of the issues. This assumption does not limit the freedom of not having a specific likelihood estimate
%(hence, a specific likelihood formula) 
for a given issue $\varphi$. To represent the absence of a specific likelihood, or an ``abstention''  on an issue $\varphi$ we use the tautology $\ell(\varphi)\geq 0$. We usually omit explicitly writing these type of formulas in the examples of judgment sets. Also, if we have $\ell(\varphi) = a\in \Jp$, we can omit including $\ell(\neg \varphi) = 1-a$ as an element of $\Jp$.

Given a finite set of $n$ information sources $N=\{1,\ldots, n\}$, a {\em likelihood profile}:
\begin{equation}
\label{eq:profile}
\Ph= (\Jp_1, \ldots, \Jp_n )\;,
\end{equation}
is a collection of  sets of likelihood judgments for an agenda $\A$, each representing one information source $k \in N$. We slightly abuse notation and write $\Jp_k \in \Ph$ to denote that $\Jp_k$ is the $k$-th likelihood judgment set in $\Ph$: 
\begin{equation}
\label{eq:judgment set in a profile}
\Jp_k=\{\ell(\varphi)\geq^* a^k_{\varphi} \mid \varphi\in \A^{\cup}\}\;,
\end{equation}
where $a^k_{\varphi}\in [0,1]$, for $k=1,\ldots,n$. The profile set of likelihood judgments that will be obtained from the information in Example~\ref{ex:running} is given in Table~\ref{tab:prof}.

We require that the  sets of likelihood judgments in the profile are {\em rational}. We now define what are rational likelihood judgments.

%\MScomment{Example here and one more table which we use in the examples with the aggregators later}

\begin{table*}[h!]
  \begin{minipage}{\textwidth}
\centering 
\includegraphics[width=0.95\textwidth]{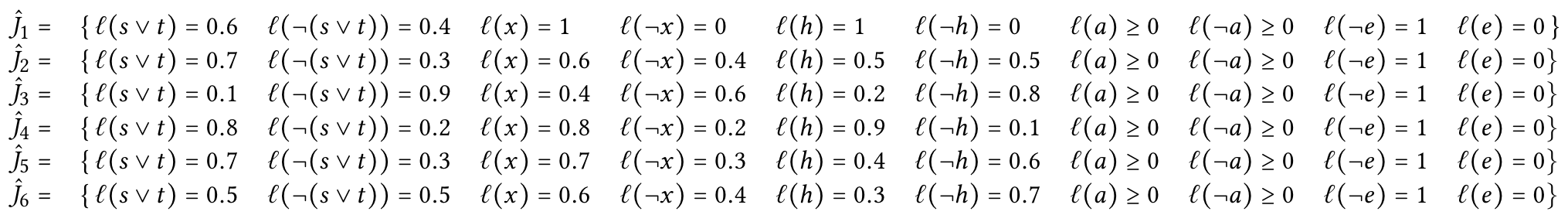}
\caption{A profile of likelihood judgments for the hotel recommendation example.} \label{tab:prof}  
\vspace{-0.2cm}
\end{minipage}
\end{table*} 

\subsection{Rationality of probabilistic judgment sets}

%\MIcomment{I don't understand comments P3C2L6, P3C2L13}

A probabilistic judgment set is  {\it consistent} if it is a consistent  set of formulas in the logic of likelihood  (according to the canonical definition of consistency). Note that a probabilistic judgment set is not always consistent. Consider, for example, the agenda $\A=\{p_1, p_1\wedge p_2, p_1\wedge \neg p_2 \}$ and  $\Jp=\{\ell(p_1)=0.5, \ell(p_1\wedge p_2)\geq 0.4, \ell(p_1\wedge \neg p_2)\geq 0.7 \}$. The set $\Jp$ is an inconsistent set of formulas, because it implies $\ell(p_1)\geq 1.1$ by the axiom (L3) of likelihood logic.  Furthermore, note that a judgment set $\Jp$ defined 
as in (\ref{eq:judgment}) has to satisfy $a_{\varphi}+a_{\neg\varphi}\leq 1$, for every $\varphi\in \A$, in order to be  consistent.

In the probabilistic case, consistency and completeness are not enough of conditions for rationality. For example, $\Jp=\{\ell(p_1)\geq 0.3, \ell(p_1\wedge p_2)\geq 0.4, \ell(p_1\wedge \neg p_2)\geq 0.1\}$ is a consistent set. However, the second formula in it implies that $\ell(p_1)\geq 0.4$, which is   stronger than the existing $\ell(p_1)\geq 0.3$ and, as such, is a more valuable judgment. 
We can formalize the notion of a stronger judgment as follows: if $\ell(\varphi) \geq^* a$ implies  $\ell(\varphi) \geq^* b$ we will say that $\ell(\varphi) \geq^* a$ is a {\it stronger judgment} than $\ell(\varphi) \geq^* b$. For example, $\ell(\varphi) = a$ implies  $\ell(\varphi) \geq a$.
To ensure that we always have the strongest possible judgments in the consistent  judgment sets, we introduce the notion of a {\em final  judgment}. A consistent probabilistic judgment set is {\it final} if it does not imply stronger judgments than the ones it contains.\footnote{{We recognize that it some cases it can be hard to check if a judgment set is final or not. In that sense, we note that this property of the judgment sets would not inflict the application of the judgment aggregation methods defined below, but it would influence the relevance of the produced results and the quality of the decision.}}
 
Probabilistic judgments can be subject to {\em probabilistic constraints} $\Cp$, where $\Cp$ is a set of likelihood formulas to denote that certain combinations of issues must have a certain likelihood. For example for agenda $\A=\{p_1, p_2, p_3\}$, where $p_1$, $p_2$, and $p_3$ represent the 
three possible states of a random variable, we can have  the integrity constraint $\ell(p_1)+\ell(p_2)+\ell(p_3)=1$. Unlike the constraints $\Ct$ which are given by the agenda setter, the probabilistic constraints $\Cp$ describe facts of the world and we assume that all information sources produce probabilistic judgment sets that are consistent with the probabilistic constraints.

A probabilistic judgment set $\Jp$ is {\it rational} if it is  complete and final, and $\Jp\cup \Cp$ is a consistent set of likelihood formulas. Given an agenda $\A$ and probabilistic constraints $\Cp$, the set of all  rational likelihood judgment sets is denoted by $\hat{\acj}$.
A profile is rational if all the judgment sets in it are rational.

We call a {\em probabilistic aggregation frame} the tuple $\paf$, where $\A$ is an agenda, $N$ is a set of  information sources,  $\Cp$ is probabilistic constraints to be satisfied by the individual judgments of the sources, and $\Ct$ are propositional constraints to be satisfied by the collective judgment. We call a {\em crisp aggregation frame} the tuple  $\caf$, but now $\Ct$ are constraints to be satisfied by the individual judgments as well.

\section{Aggregating likelihood judgments} 
\label{sec:Probabilistic aggregation functions}

%There are several types of JA functions that can be defined. First, as in ``classical'' JA, we distinguish between resolute and irresolute aggregators \cite{LangPSTV15}: a resolute aggregator  assigns a unique judgment set to a profile of judgments, while an irresolute aggregator assigns a {\em set of judgment sets} to a judgment profile. Second,

We distinguish between {\em crispifying aggregators} and {\em direct aggregators}. Crispifying aggregators first aggregate the likelihood profile into likelihood judgment set(s)   and then use given threshold values to  ``crispify'' these sets.   Direct aggregators assign a crisp judgment set (or sets) to a likelihood profile directly. 
 
%that assign crisp judgment set(s) directly to the profile. 

%Generally in JA, resolute aggregators tend to be subject to Arrowian  \cite{Arrow:63}  impossibility results. Namely consider the ``standard'' JA properties\footnote{These will be discussed further in the section in more detail.}: universal domain, non-dictatorship, unanimity and systematicity \cite{LangPSTV15,GrossiP:2014}.  It is simple to extend these properties to the likelihood judgment framework and show that no resolute aggregation function (with or without a crispifying vector) exists that satisfies all these properties\footnote{Details of that proof can be found here https://anon.to/wtelFr.}\MScomment{P3C2L6 mislam deka go mrzelo da go chita linkot. Ovde ja bi da gi vratime definiciite shto se vo apendix nazad} In particular, the ``problematic" property is systematicity which requires that the collective judgment on an issue depends only on the individual judgments on that issue (independence) and that all issues are aggregated in the same way (neutrality). 
%{\em Because we are interested in defining concrete  and non-trivial aggregators, we focus on irresolute aggregation functions.}{\MScomment{P3C2L13: Since issues are not symetric why is it reasonable to require systematicity. Isn't neutrality closure to negating all reports on an issue. Mislam oti pojma nema sho zbroi, ama sum sigurna deka jas pojma nemam shto zboram}} 

The rest of this section is organized as follows: We first consider details of the ``crispification step'' and introduce the formal definition of an aggregator; 
%then in Section \ref{subsec:properties} we define the properties that a likelihood operators can satisfy; 
in Section \ref{subsec:comparison} we propose a way to compare the likelihood aggregators with the classical ones; and finally in Section \ref{subsec:Crispifying aggregators} and Section \ref{subsec:direct aggregators} we introduce several likelihood aggregators and analyse %their properties and
connections with the corresponding classical ones.
%before giving the formal definitions of likelihood aggregators.  
\subsection{Crispifying}
\label{subsec:crispifying}

Given a probabilistic judgment, we can obtain a crisp judgment  by choosing a  threshold coefficient $\q\in [0,1]$.  This coefficient can be by default set to 0.5 for each issue, but otherwise we assume that it is specified by the agenda setter, in response to the question: How likely should an issue be in the least in order to be accepted as true?
 We define the  judgment crispifying function $\textrm{crisp}()$ as follows: 

\begin{equation}\label{eq:jcrisp}
\textrm{crisp}(\ell(\ai) \geq^* a, \q)=\left\{
\begin{array}{ll}
\{\ai\}, &\textrm{ if } a \geq\q\\
\emptyset, & \textrm{ otherwise } 
\end{array}
\right.
\end{equation}
%\MIcomment{zabeleska od Lian: $\varphi$ i $\emptyset$ ne se vo ista kategorija, i.e. ne bila dobra definicijata. Moze da stavime eventualno $\{\varphi\}$, ama neznam..}
According to  the above definition, if a likelihood judgment on a statement $\ai\in\A^{\cup}$ has a (minimal) likelihood strictly greater than or equal to $\q$, we assign it a Boolean judgment $\ai$. 
Otherwise, no Boolean judgment is assigned for this issue. 
If we decide to be strict on an issue $\varphi$ and accept it only if true, we set $c_{\varphi}=1.$

We can crispify a probabilistic judgment set $\Jp$ by crispifying each of its judgments. We distinguish between {\em issue-wise   crispifying} when a different coefficient is assigned for every agenda issue and {\em uniform crispifying} when the same coefficient is used for every agenda issue.

Let $\qv = (\q_{\varphi})_{\varphi\in\A^{\cup}}$ be a vector of coefficients, where each $\q_{\varphi}\in [0,1]$, and $\q_{\varphi}+\q_{\neg\varphi}> 1$. We call this a vector of crispifying coefficients. A judgment set crispifying is defined as follows:
\begin{equation}
\label{eq:crispifying}
\textrm{crisp}(\Jp, \qv)=
\bigcup \{\textrm{crisp}(\ell(\ai) \geq^* a_{\varphi}, \q_{\varphi}) \mid \ell(\ai) \geq^* a_{\varphi} \in \Jp \}
\end{equation}

The condition $\q_{\varphi}+\q_{\neg\varphi}> 1$, along with the  consistency requirements $a_{\varphi}+a_{\neg\varphi}\leq 1$, assures that only one element of the set $\{\varphi, \neg\varphi\}$ is in $\textrm{crisp}(\Jp, \qv)$ for each $\varphi\in\A$. If $c_{\varphi}=c$, for every $\varphi\in\A^{\cup}$, and some $c\in (1/2,1]$, the crispifying defined by Eq.(\ref{eq:crispifying}) is uniform, and we denote it by $\textrm{crisp}(\Jp, c)$. Note that the constraint $c>1/2$ follows from the consistency requirement on the crispifying vector.

Observe that the obtained crisp set of judgments may be incomplete. Further, we allow the agenda setter to freely choose whichever coefficients she wants for any of the issues,  depending on the given context. This freedom of choice is done here for simplicity. We can, however, argue that a freely chosen vector of crispifying coefficients may be seen as imposing a certain level of independence on the issues.  We can argue that if for two issues $\varphi_1,\varphi_2 \in \A$ when it holds that $\{\varphi_1 \rightarrow \varphi_2\} \in \Ct$, then it should not be allowed that $c_1 < c_2$, \ie in this case we would need the additional constraint $c_1\geq c_2$. Also, if $\varphi_1 \leftrightarrow \varphi_2 \in \Ct$, we would need to have $c_1=c_2$, \ie logically equivalent issues should have the same likelihood threshold requirement. Restricting the values in $\qv$ with respect to $\Ct$ is an interesting aspect of our framework and is a line of future work we intend to pursue.

We now give a formal definition for an aggregator.

 \begin{definition}
 \label{def:aggregator}
 Let $\paf$ be a probabilistic aggregation frame and let $\app$ be the set of all rational likelihood profiles for it, while  $[0,1]^{2m}$, for $m=|\A|$,  is  the set of all crispifying vectors $\qv$. Let $\hat{f}$ be a mapping from  $\app$ to $\hat{\acj}$.
 A crispifying judgment aggregation function $\hat{F}$ is a mapping from $\app \times[0,1]^{2m}$  to  $\mathcal{P}(\acj)$, i.e. $\Fh (\Ph, \qv)\subseteq \acj$, where $\Fh (\Ph, \qv)= \textrm{crisp}(\hat{f}(\Ph), \qv)$ is the classical judgment set obtained by crispifying the likelihood judgment set $\hat{f}(\Ph)$. A direct judgment aggregation function $\F$  is a mapping from $\app$   to  $\mathcal{P}(\acj)$, i.e. $\F (\Ph)\subseteq \acj$.
 \end{definition}
{According to the above definition an aggregator is defined for {\em every} rational profile and {\em always} produces {\em rational} judgment sets as a result, properties that are later introduced as {\em universal domain} and {\em rationality}, correspondingly.  We embed these properties in the definition since they are the most basic desirable properties of the aggregation process, usually satisfied by design. However, while the universal domain is satisfied by all the aggregators defined below, we sometimes deviate from Definition \ref{def:aggregator} by defining some aggregators that are not rational.
 
 Notice also that, even thought we insist on the collective judgment being crisp, in every crispifying aggregator (and, implicitly, in many direct aggregators) an intermediate probabilistic aggregate is available if needed in the decision process.}

\subsection{Classical vs probabilistic aggregators}
\label{subsec:comparison}

 Same as they do  in  \cite{LangPSTV15}, we define a classical irresolute aggregation function $\bar{\Agg}(P)$  as one that  maps each crisp rational profile  $P$ of judgments to a nonempty set of crisp rational judgment sets.  

Consider a crisp judgment set $\Js \in \acj$. We define its  corresponding  probabilistic judgment set $J^*$ in the following way:   
 
\begin{equation}
 \Js^*=\{\ell(\varphi)= 1, \ell(\neg\varphi)=0 \mid \ai \in \Js\}
 \end{equation}
Note that $\textrm{crisp}(\Js^*, \qv)=\Js$ for every vector of crispifying coefficients $\qv = (\q_{\varphi})_{\varphi\in\A^{\cup}}$  such that $c_{\varphi}=1$,  for $\varphi\in\Js$. 

Given a crisp profile $P=(\Js_1,\ldots,\Js_n)$ we define %$P^*$ as 
$$P^* = (\Js^*_1,\ldots, \Js^*_n),$$
to be its correspondent probabilistic  profile. 
We can now define what it means for a likelihood aggregator to generalize a crisp aggregator. 
\begin{definition}  Let $\paf$ be a probabilistic aggregation frame. Consider the corresponding crisp frame $\caf$ and let $\acp$ be the set of all rational likelihood profiles for it. Let   $P^*$ be a corresponding profile for a $P \in \acp$. A direct likelihood aggregator $F$ generalizes a crisp aggregator $\bar{\Agg}$ if $\bar{F}(P) = F(P^*)$ for each $P \in \acp$. A crispifying likelihood aggregator $\hat{F}$ generalizes a crisp aggregator $\bar{\Agg}$ if there exists $\qv \in [0,1]^{2m}$ such that $\bar{F}(P) = \hat{F}(P^*, \qv)$ for each $P \in \acp$. 
\end{definition}

%\MIcomment{The above definition enables us to transfer the impossibility result from the classical judgment aggregation framework to our likelihood framework. Namely, if we suppose that there exists a likelihood aggregation rule that satisfies ........ then the restriction of this rule to the set of classical zero, one likelihood judgments will define a classical aggregation rule that satisfies the same properties, and this is in contradiction to the classical impossibility result.}

%We compare the aggregators we define with crisp aggregators and show when ours generalize them.  Due to space restrictions we cannot give the definitions of crisp aggregators, and we refer the reader to \cite{LangPSTV15} for these.  \MIcomment{mislam posledniot pasos cel da go brisime, prvata recenica ja napisav prethodno vo vovedot, a od [19] dobro ke e da gi napisime definiciite, ne se mnogu, imame mesto.}

\subsection{Crispifying aggregators}
\label{subsec:Crispifying aggregators}

We now consider two classes of crispifying aggregators. 
\paragraph{Uniform quota aggregators}
 Quota aggregators assign a crisp judgment to elements in $\A^\cup $ in two steps.  First, the collective likelihood of $\varphi$ is assigned. The collective likelihood for $\varphi$ is the maximal $a\in [0,1]$ such that the  number of agents in the profile who assign a likelihood of at least $a$ reaches a given quota $q$.  Second, the collective likelihood judgments are crispified using a crispifying coefficient.    The formal definition follows.

\begin{definition} Given a profile $\Ph$, a crispifying vector $\qv$ and a quota $\qu\in \mathbb{N}$, $1 \leq \qu\leq n$,  we define the {\em uniform quota function} $\hat{f}_{\qu}$:
 %\F_{\qu}(\Ph,\qv )=\{\varphi: |\{k: a^k_{\varphi} > \q_{\varphi}\}|>q\}\;.
\begin{equation}
\label{eq:uniform quota function}
\begin{array}{l}
%\Fh_{\qu}(\Ph)= \{\ell(\varphi)\geq^* a_\varphi\;: \max_{0\leq a_\varphi \leq 1} (|\{k: a^k_{\varphi} > a_{\varphi}\}|>q)\}\;.\\
\hat{f}_{\qu}(\Ph)= \{\ell(\varphi)\geq^* a_\varphi\;: a_\varphi=\max_{0\leq a\leq 1} (|\{k: a^k_{\varphi} \geq a\}|\geq q)\}\;.\\
\Fh_{\qu}(\Ph, \qv) = \textrm{crisp}(\hat{f}_{\qu}(\Ph),\qv).
\end{array}
\end{equation}
\end{definition}

As an illustration, consider the example in Table~\ref{tab:prof}. For  a uniform $\qv=(0.6, \ldots, 0.6)$ and a quota $\qu = 3$ we obtain $\Fh_{3}(\Ph,0.6)  = \{ s \vee t,  x, \neg  h,  \neg e\}$, which is inconsistent with  $\Ct$. %(Note that, since in this example the judgments provide exact probabilities, the probabilities of the negations easily follow from the given ones and the assumed consistency of the judgment sets.)

If $\qu = n$, we obtain the unanimous function that selects as collective only those judgments  $\ai$ who are assigned a likelihood $a_{\varphi}^k\geq\q_{\varphi}$  by all the agents $k$.  For $\qu=\lfloor\frac{n}{2}\rfloor+1$  we obtain the {\em issue-by-issue majority function}, which we denote with $M$.  Under issue-by-issue majority function the profile is aggregated by selecting the judgments that are in the most (more than a half) of the judgment sets in the profile.  The set $M(\Ph,\qv)=\Fh_{\lfloor\frac{n}{2}\rfloor+1}(\Ph,\qv)$ is called a {\em majoritarian set} for $\Ph$ and $\qv$.

The majoritarian set of a crisp profile $P$  is denoted $m(P)$ and  contains
all the elements of  $\A^\cup$ that are supported by a strict majority of the individual judgment sets:
\begin{equation}
  m(P)=\{\varphi:\varphi\in \A^\cup, |\{k:\varphi\in J_k\}|> \frac{n}{2}\}
\end{equation}
The following theorem can be easily proved:
\begin{theorem}
Let $\paf$ be a probabilistic likelihood frame, and $\qv=(c_{\varphi})_{\varphi\in \A^{\cup}}$ be a vector of coefficients such that $c_{\varphi}>0$, for every $\varphi\in \A^{\cup}$. Then $m(P)=M(P^*,\qv)$, where $P^*$ is the corresponding probabilistic profile to the crisp profile $P$.
\end{theorem}

\paragraph{Aggregators based on the majoritarian set}
One way to aggregate probabilistic judgments into a rational crisp judgment set is to minimally modify the  set $M(\Ph,\qv)$  so that it becomes consistent with $\Ct$.  This approach is used in crisp judgment aggregation to define several aggregators {\em based on the majoritarian set} \cite{LangPSTV15}. We can extend the definition of aggregators based on the majoritarian set to likelihood aggregators as follows. 

\begin{definition}\label{def:mba}
A crispifying likelihood aggregator $\hat{F}$ is based on the majoritarian set if for every $\Ph, \Ph'\in \app$ it holds that $\hat{F}(\Ph,\qv) =\hat{F}(\Ph',\qv)$ if  $\hat{f}_{q}(\Ph) = \hat{f}_{q}(\Ph')$, where $q=\lfloor\frac{n}{2}\rfloor+1$ and $n$ is the number of agents in $\Ph$ and $\Ph'$.
\end{definition}

Since classical  aggregators  based on the majoritarian set use as an input not the entire profile but just the set of majority judgments their definitions can be easily extended to handle profiles of probabilistic judgments as well.  Proposition~\ref{prop:allsame} proves that the latter is not necessary. 

First, recall the uniform quota rule for ``classical'' JA \cite{GrossiP:2014}. For profiles $\Pc = (\Js_1, \ldots, \Js_n )$ of crisp judgment sets, the crisp uniform quota function $U_{\qu}$ is defined to give as output the set of those judgments that are in at least $\qu$ judgment sets in $\Pc$:

\begin{equation}
    U_{\qu}(P)=\{\varphi:\varphi\in \A^{\cup}, |\{k:\varphi\in J_k\}|\geq \qu\}
\end{equation}
 
Let  $\Pc_\qv = (\textrm{crisp}(\Jp_1, \qv), \ldots, \textrm{crisp}(\Jp_n, \qv) )$ be the profile obtained by crispifying each probabilistic judgment set in 
a $\app$-profile $\Ph=(\Jp_1,\ldots, \Jp_n)$ by a vector $\qv $. We show that first calculating $\hat{f}_{\qu}(\Ph)$ and then crispifying is the same as first crispifying each judgment sets in the profile into $\Pc_\qv$  and then  applying  $U_{\qu}$ to this $\Pc_\qv$. %  the probabilistic quota function $\hat{F}_{\qu}$ to  a profile $\Ph$ and coefficients $\qv $ is the same as applying 
{Namely, we show that  } {\bf the quota function commutes with the crispifying function.}

\begin{proposition}\label{prop:allsame}  For every $\Ph \in \app$, crispifying coefficients $\qv $, and quota $\qu \leq n $ it holds that $
\Fh_{\qu}(\Ph,\qv ) =  U_{\qu}(\Pc_\qv).$
\end{proposition}
\begin{proof}
We prove that $\ai \in \Fh_{\qu}(\Ph,\qv )$ iff $\ai \in U_{\qu}(\Pc_\qv)$.
Consider $\ai\in \A$. The proof is similar for $\neg \ai \in \A^\cup$. 
We have that  $\ai \in \Fh_{\qu}(\Ph,\qv )$ iff there exists at least $\qu$ agents $k$ s.t. $a^k_{\varphi} \geq \q_{\varphi}$. This is the case iff there are at least $\qu$ agents $k$ in $\Pc_\qv $ s.t. $ \ai \in \Js_k$. Thus necessarily  $\ai \in \hat{f}_{\qu}(\Pc_\qv )$ and we get $\ai \in \Fh_{\qu}(\Ph,\qv )$ iff $\ai \in U_{\qu}(\Pc_\qv)$. 
\end{proof}
 
{\bf Proposition~\ref{prop:allsame} shows that we can use classical aggregators based on the majoritarian set to aggregate likelihood judgments.}
We simply crispify the profile first and then apply the classical aggregator.  As a consequence, however,  we can conclude that finding the collective judgments for probabilistic profiles is as computationally hard as for crisp profiles when these aggregators are used. Complexity results for these aggregators are given in \cite{LangECAI14,EndrissDeHaanAAMAS2015}.
%\MIcomment{P5C2-TWO COMMENTS} \MScomment{to be resolved by giving the decision problem and formulating complexity properties as Property.}

 Let us consider the  weighted majoritarian aggregation rules defined in \cite{LangPSTV15}. These rules, in addition to using a (crisp) majoritarian set as input, also use the number $N(\Pc,\varphi)$ of agents that support each judgment in that majoritarian set:
\begin{equation}
\label{eq:number of agents}
N(\Pc,\varphi)=|\{k:\varphi\in \A^{\cup}, \varphi\in J_k\}|\;.   
\end{equation}

In general, according to \cite{LangPSTV15}, a classical irresolute aggregation function $\bar{F}$ is  based on the weighted majoritarian set if for every two JA-profiles $P$ and $P'$, $N(P,\varphi) = N(P^{\prime},\varphi)$ implies $\bar{F}(P) = \bar{F}(P')$, for every $\varphi \in \A^\cup$. 

An example of such an aggregator is the {\em median rule} of \cite{LangPSTV15}. We give the definition of this aggregator using our notation: 
\begin{equation}\label{eq:medc}
    \textrm{\textsc{med}} (\Pc) =  \argmax{J \in \acj} \sum_{\varphi \in J} N(\Pc,\varphi)\;,
\end{equation}
where $N(\Pc,\varphi)$ is defined as in Eq.(\ref{eq:number of agents}).

Proposition~\ref{prop:allsame} shows that the weighted majoritarian set  can also be directly used to aggregate likelihood judgments. This can be done by generalizing the definition of $N(P,\varphi)$. 

Let us define  $\hat{N}(\Ph,\varphi, c)$ to be the number of agents that assign to $\varphi$ a likelihood greater than or equal to some  $c\in[0,1]$  in the profile $\Ph$:
\begin{equation}
\hat{N}(\Ph,\varphi, c)=|\{k:a^k_{\varphi}\geq c\}|\;.   
\end{equation}

When  each of the judgment sets in the profile $\Ph$ is  crispified by a vector of coefficients $\qv=(c_{\varphi})_{\varphi\in\A^{\cup}}$, such that $c_{\varphi}=c$,  then $\hat{N}(\Ph,\varphi, c)$ is exactly $N(\Pc,\varphi)$ for the resulting crisp profile $\Pc$.

%\MIcomment{A THEOREM OR A PROPOSITION HERE?}

%\MScomment{Sure why not.}
 
 However, we do not have to constrain ourselves with just using $\hat{N}(\Ph,\varphi, c)$, we can further generalize the weighted majoritarian rules of \cite{LangPSTV15} to consider not only how many agents assigned a likelihood over the threshold but also the likelihoods they do assign. This is one of the ways in which we can obtain direct aggregators.

 \subsection{Direct aggregators}
 \label{subsec:direct aggregators}
 
 Let us consider again the   {\em median rule} of \cite{LangPSTV15}. We can  define the  {\em median likelihood aggregator} to generalize the median rule. 
 
 \begin{definition} Given a profile $\Ph$, the median likelihood aggregator  is defined as
\begin{equation}
\widehat{\textrm{\textsc{med}}}(\Ph) = \argmax{J \in \acj}  \sum_{\varphi \in J} S(\varphi, \Ph)\;,
\end{equation}
where 
\begin{equation}S(\varphi, \Ph)=\sum_{J_k\in \Ph} a^k_{\varphi}\end{equation}
\end{definition}

The median likelihood aggregator assigns to a given profile the classical judgment set that gives the maximum sum of likelihoods assigned by all the agents to all the issues in it.
\Eg the outcome using $\widehat{\textrm{\textsc{med}}}$ for the Example~\ref{ex:running} profile is the crisp judgment set $\{ s \vee t,  h,  x, \neg e\}$ with a ``score''  $\sum_{\varphi \in J} S(\varphi, \Ph)$ of 16.8.

%\MScomment{Tabelata treba odnovo da se presmeta}
\begin{table}[h!]
    \centering
    \begin{tabular}{|ccccc|c|c|}\hline
    \multicolumn{5}{|c|}{$J$} & & \\
    $s\vee t$ & $x$ & $e$ & $h$ & $a$ &$\sum_{\varphi \in J} S(\varphi, \Ph)$ &  $D_e(\Js, \Ph)$\\ \hline
    0 & 0 & 0& 0 &1 &13,2 & 11.00837576 \\
  %  0 & 0 & 0& 1& \\
     0 & 0 & 1& 0 &0 & 7.2 & 13.98315881\\
  %  0 & 0 & 1& 1& \\
     0 & 1 & 0& 0 &1 &15.4& 13.09902931\\
   % 0 & 1 & 0& 1&   \\
     0 & 1 & 1& 0 &1 &9.4& 9.935132833\\
  %  0 & 1 & 1& 1 & \\
      %1 & 0 & 0& 0 & \\
    1 & 0 & 0& 1& 1&14,6& 10.45417943 \\
     1 & 0 & 1& 0&0 & 8& 13.7033258\\
    1 & 0 & 1& 1& 0 & 8.6& 13.48039101\\
    % 1 & 1 & 0& 0& \\
    1 & 1 & 0& 1&1 & \textbf{16.8}& \textbf{8.91436323} \\
  %   1 & 1 & 1& 0& \\
    1 & 1 & 1& 1& 1& 10.8& 12.38906945\\\hline
    \end{tabular}
    \medskip 
    \caption{The set of all rational crisp judgment sets for the frame in Example~\ref{ex:running}, their ``scores'' and Euclidean distance to the likelihood profile.}
    \label{tab:mylabel}
    \vspace{-0.2cm}
\end{table}

Recall that, for a crisp profile $\Pc$, the median rule is defined as in Eq.(\ref{eq:medc}). Proposition~\ref{prop:med} is straightforward.

\begin{proposition}\label{prop:med} $\textrm{\textsc{med}} (\Pc)  = \widehat{\textrm{\textsc{med}}}(\Pc^*)$.
\end{proposition}

%\MIcomment{PROPERTIES FOR THE MEDIAN RULE}

We now define three classes of direct  aggregators. 

\paragraph{Sequential direct aggregators.}
An intuitive way to define direct aggregators is to aggregate the judgments issue-by-issue in a sequence  by first ``settling'' the judgment on the issue for which the agents have assigned the highest likelihood. To do this, we need to define what it means for a judgment to have ``the highest likelihood'' in a profile. Several options exist and each of them leads to a different aggregator.  We consider only one here,  in order to illustrate the process.

We define ``the highest likelihood'' to be the highest {\bf average} likelihood assigned to a judgment in a profile.

\begin{definition} [Average likelihood]
 Given $\ai\in \A^{\cup}$  and a profile $\Ph$, the average  likelihood for $\ai$ in $\Ph$ is defined as 
\begin{equation} 
\label{eq:average likelihood}
\avg{\ai} = \frac{1}{n}\sum_{k=1}^n  a_{\varphi}^k,\end{equation}
  \end{definition}

Note that since, in general, we have likelihood judgments with  inequalities, these average likelihoods are actually average minimal likelihoods.  Equivalently, we could have a vector of average maximal likelihoods taking $1-a_{\neg\varphi}$ instead of $a_{\varphi}$ for every $\ai$. We could possibly consider linear weighted average or any other opinion pooling function to define $\avg{\ai}$ in Eq.(\ref{eq:average likelihood}).

Let $\mathbf{a}_{\Ph}$ be the vector of average likelihoods assigned to each $\ai\in \A^{\cup}$ given a profile $\Ph$.  Namely 
\begin{equation}
\mathbf{a}_{\Ph}=(\avg{\ai})_{\ai\in\A^{\cup}}.
\end{equation}

 The sequential average aggregator $\Fq$  builds a crisp collective judgment set  sequentially, adding first as many as possible of  the judgments with highest average likelihood then moving on to judgments with the  next highest average likelihood  and adding them  only if they are consistent with the already added judgments (skipping them otherwise). 

\Eg for the profile in Table~\ref{tab:prof} we have the following:
$$\mathbf{a}_{\Ph}=(\underset{s\vee t}{\underbrace{0.56, 0.44,}} \underset{h}{\underbrace{  0.55, 0.45,}}   \underset{x}{\underbrace{ 0.75, 0.25,}} 
\underset{a}{\underbrace{ 0.0, 0.0,}}
\underset{e}{\underbrace{ 1.0,0.0}}).$$   We  obtain $\Fq (\Ph) = \{\neg e, x, s \vee t, h, a\}$,  with the judgments written  in the order in which they were added. If instead we had $$\mathbf{a}_{\Ph'}=(\underset{s\vee t}{\underbrace{0.56, 0.44,}} \underset{h}{\underbrace{ 0.3, 0.7,}}   \underset{x}{\underbrace{0.75, 0.25,}} 
\underset{a}{\underbrace{ 0.0, 0.0,}}
\underset{e}{\underbrace{ 1.0,0.0}}  )$$
for some profile $\Ph'$, after adding $\{\neg e, x, \neg h\}$ next we should have had to add  $s \vee t$ because its average likelihood is $0.56$. But since $\{\neg e, x, \neg h, s \vee t, a \}$ is not consistent with $\Ct$, we would obtain   $ \Fq (\Ph) = \{\neg e, x, \neg h, \neg (s \vee t), a\}$.  

For likelihood profiles $\Pc^*$ corresponding to a crisp profile $\Pc$, we have that 
\begin{equation}
    \label{eq:leximax}
\Fq (\Pc^*) =  \textrm{{\sc leximax}} (\Pc)\;, 
 \end{equation}
  where {\sc leximax} is the non-probabilistic judgment aggregation rule   defined in \cite{NehringPivato2011,EKM13}. We omit the definition of {\sc leximax}  here due to space issues and the triviality of the proof of Eq.(\ref{eq:leximax}).

Many different functions can be defined using the average likelihood. The immediate approach would be to build aggregators inspired by the class of scoring rules \cite{Dietrich:2013}. Furthermore, we can work with not only the mean but also with max, min or otherwise polled individually assigned likelihoods.

Next we focus on the class of distance-based aggregation functions. 
 \paragraph{Distance-based aggregation}
Distance-based aggregators aggregate profiles by considering all possible collective outcomes and choosing the one that is ``most similar'' to the profile at hand. Similarity is defined by a distance measure - the greater the distance between two judgment sets, the less similar they are. Distance from a profile to an outcome (judgment set) is defined as the sum or the maximum of the distances between the outcome and each of the judgment sets to the profile. Thus, to define a direct aggregator using the distance-based approach, we need to define a distance from a crisp judgment set to a likelihood judgment set.  To do this, recall that a crisp judgment can be seen as special case of a likelihood judgment.

%Consider a crisp judgment set $\Js \in \acj$. We define its correspondent probabilistic judgment set $J^*$ in the following way:   $\ai \in \Js$ iff $\{\ell(\varphi)= 1, \ell(\neg\varphi)=0 \}\subseteq \Js^*$ and ;  $\ai \not\in \Js$ iff $\{\ell(\varphi)= 0, \ell(\neg\varphi)=1 \} \subseteq \Js^*$, for every $\ai\in\A^\cup$.

Given a distance function $d$ (defined over vectors of reals) we can define a distance-based aggregation function $\hat{F}_d$ as 

\begin{equation}\label{eq:gendist}
\hat{F}^{d, \Sigma}(\Ph) = \argmin{J \in \acj} \sum_{\Jp_k \in \Ph} d(J^*,\Jp_k)\;,
\end{equation}
where the distance between two judgment sets is defined as:
\begin{equation}\label{eq:gendist1}
d(\Jp,\Jp_k)=d((a_{\varphi})_{\varphi\in\A^\cup},(a^k_{\varphi})_{\varphi\in\A^\cup})\;.
\end{equation}
Alternatively instead of sum we can use max:  
\begin{equation}\label{eq:eucmax}
\hat{F}^{d, \max}(\Ph) = \argmin{J \in \acj} \max_{\Jp_k \in \Ph} d(J^*,\Jp_k).
\end{equation}

When we take $d$ to be the Euclidean distance, (\ref{eq:gendist}) becomes
\begin{equation}\label{eq:eucd}
\hat{F}^{e, \Sigma}(\Ph) = \argmin{J \in \acj}\; D_e(\Js, \Ph),
\end{equation}
where
\begin{equation}\label{eq:eucd2}
D_e(\Js, \Ph)=  \sum_{\Jp_k \in \Ph} \sqrt[]{\sum_{\ai\in\A^\cup} (a_{\varphi}^k - a_{\varphi})^2}.
\end{equation}

 \Eg the outcome of applying the rule in Eq.(\ref{eq:eucd}) to the profile in the Example~\ref{ex:running} is the crisp judgment set $\{ s \vee t,  h,  x, \neg e,a\}$ 
at a distance 9.17 from the profile, see Table~\ref{tab:mylabel}.

Numerous statistical distance measures can be used,  \cite{deza2009} offers variety of examples. Further research is needed to establish what distance measure is a good choice. 

In ``classical'' judgment aggregation, the distance-based aggregator, also known as the Kemeny rule \cite{LangPSTV15} is defined  as follows. Let $d_H$, the Hamming distance, between two crisp judgment sets $\Js_1$ and $\Js_2$ on the same crisp frame $\caf$ be defined as the number of judgments on which  $\Js_1$ and $\Js_2$ differ. For example, for $\Js_1 = \{p, p\rightarrow q, q\}$ and $\Js_2=\{ \neg p, p\rightarrow q, \neg q\}$ we have  $d_H(\Js_1, \Js_2)=2$. The Kemeny rule, for a given $\acj$ is defined as 

\begin{equation}F^{d_H, \Sigma}(\Pc) = \argmin{\Js \in \acj} \sum_{\Js_k\in P} d_H(\Js,\Js_k).\end{equation}

We can however  observe that the Euclidean distance is the same as the Hamming distance when the likelihood judgment values are in $\{0,1\}$ and thus we obtain that Proposition~\ref{prop:kemeny} holds.

\begin{proposition}\label{prop:kemeny}
$F^{d_H, \Sigma}(\Pc)  = \hat{F}^{e, \Sigma}(\Pc^*).$
\end{proposition}

If we use the $L1$ distance (namely, the sum of  differences between absolute values issue-by-issue), we obtain another generalization of the Hamming distance.
 
It is well known that for every crisp frame $\caf$ it holds that $\textrm{med}(\Pc) = F^{d_H, \Sigma}(\Pc)$, see for example \cite{LangPSTV15} for a formal proof. This relationship does not extend to $\hat{F}^{e, \Sigma}$ and  $\widehat{\textrm{\textsc{med}}}$. We give a counter example. Consider the profile on Table~\ref{tab:counter1}. For this profile,  $\Cp = \top$ and $\Ct=\top$.  All the rational judgment sets for $\A = \{p, p\rightarrow q, q\}$, $\Cp = \top$ and $\Ct=\top$ are given in Table~\ref{tab:counter2}. We have that $\widehat{\textrm{\textsc{med}}}(\Ph)= \{p, \neg(p\rightarrow q), \neg q\}$, while $\hat{F}^{e, \Sigma}(\Ph) = \{p, p\rightarrow q, q\}.$
%\begin{table}[h!]
%\centering
%\begin{tabular}{r|cccccc}
%& $p$ & $\neg p$ & $p \rightarrow q$ & $\neg (p \rightarrow q)$ &$q$ & $\neg q$ \\ \hline
%$\Jp_1$   & 0.0 & 0.3  & 0.8 & 0.1 & 0.6  &0.2\\
%$\Jp_2$ & 0.1 & 0.4  & 0.4 & 0.2 & 0.3  &0.6\\
%$\Jp_3$ & 0.8 & 0.0& 0.1 & 0.8 & 0.3  &0.7\\ \hline
%\end{tabular}
%\caption{An example of a judgment aggregation.} \label{tab:counter1} 
%\end{table}

\begin{table}[h!]
\centering
\begin{tabular}{r|cccccc}
& $p$ & $\neg p$ & $p \rightarrow q$ & $\neg (p \rightarrow q)$ &$q$ & $\neg q$ \\ \hline
$\Jp_1$   & 0.0 & 0.3  & 0.8 & 0.1 & 0.6  &0.2\\
$\Jp_2$ & 0.1 & 0.4  & 0.5 & 0.2 & 0.3  &0.6\\
$\Jp_3$ & 0.8 & 0.0& 0.1 & 0.8 & 0.3  &0.7\\ \hline
\end{tabular}
%\medskip
\caption{An example of a likelihood profile} \label{tab:counter1} \vspace{-0.2cm}
\end{table}

%\begin{table}[ht!]
%    \centering
 %   \begin{tabular}{|ccc|c|c|}\hline
  %  \multicolumn{3}{|c|}{$J$} & & \\
   % $p$ &  $p\rightarrow q$  & $q$ &  $\sum_{\varphi \in J} S(\varphi, \Ph)$& %$D_e(\Js, \Ph)$ \\ \hline
   % 1 &  1 &  1&  3.4 &  4.219\\
    %1 & 0 & 0  & \underline{3.5} &\underline{3.8189}\\
    %0 & 1 & 1  &3.2&4.192\\
    %0 & 1 & 0 &\underline{3.5} &4.106\\\hline
  %  \end{tabular}
   % \caption{The set of all rational crisp judgment sets for $\A = \{p, %p\rightarrow q, q\}$, $\Cp = \top$ and $\Ct=\top$, their ``scores" and %Euclidean distance to the likelihood profile from %Table~\ref{tab:counter1}.}
%    \label{tab:counter2}
%    \vspace{-0.5cm}
%\end{table}

\begin{table}[ht!]
    \centering
    \begin{tabular}{|ccc|c|c|}\hline
    \multicolumn{3}{|c|}{$J$} & & \\
    $p$ &  $p\rightarrow q$  & $q$ &  $\sum_{\varphi \in J} S(\varphi, \Ph)$& $D_e(\Js, \Ph)$ \\ \hline
    1 &  1 &  1&  3.5 &  4.1818\\
    1 & 0 & 0  & 3.5 &\underline{3.8537}\\
    0 & 1 & 1  &3.3&4.1064\\
    0 & 1 & 0 &\underline{3.6} &4.0032\\\hline
    \end{tabular}
   % \medskip
    \caption{The set of all rational crisp judgment sets for $\A = \{p, p\rightarrow q, q\}$, $\Cp = \top$ and $\Ct=\top$, their ``scores'' and Euclidean distance to the likelihood profile from Table~\ref{tab:counter1}.}
    \label{tab:counter2}
    \vspace{-0.2cm}
\end{table}

%This equality still holds for the likelyhood generalizations  $\hat{F}^{e, \Sigma}$ and  $\widehat{\textrm{\textsc{med}}}$. 

%\begin{theorem}
%For every $\paf$, it holds that 
%$$ \widehat{\textrm{\textsc{med}}}(\Ph) = \hat{F}^{e, \Sigma}(\Ph)  $$
%\end{theorem}
%\begin{proof}
%We need to show that for two crisp judgment sets $\Js$ and $\Js'$ and a likelihood profile $\Ph$ it holds:
%\begin{equation}\label{eq:ms}
%  \sum_{\varphi \in J} \sum_{J_k\in \Ph} a^k_{\varphi} \geq   \sum_{\varphi \in J'} \sum_{J_k\in \Ph} a^k_{\varphi}
%\end{equation}
%
%
%if and only if 
%\begin{equation}\label{eq:euc}
%\sum_{\Jp_k \in \Ph} \sqrt[]{\sum_{\ai} (a_{\varphi}^k - a_{\varphi})^2} \leq % \sum_{\Jp_k \in \Ph} \sqrt[]{\sum_{\ai} (a_{\varphi}^k - a'_{\varphi})^2}
%\end{equation}
%Let us take the value
%\begin{equation}\label{eq:abs}
%\sum_{\Jp_k \in \Ph} \sum_{\varphi}\frac{|a^k_{\varphi} - a_{\varphi}|}{2} \leq \sum_{\Jp_k \in \Ph} \sum_{\varphi}\frac{|a^k_{\varphi} - a'_{\varphi}|}{2}
%\end{equation}

%It can be shown that (\ref{eq:euc}) iff (\ref{eq:abs}) and (\ref{eq:abs}) iff (\ref{eq:ms})? Tochno li e ova? 

%\end{proof}}
 The relationship between the distance-based aggregator and the median rule is broken when  the $L1$ distance is used as well and a counter example is not difficult to be found. This is because the relationship between the judgment on $\varphi$ and the judgment on $\neg \varphi$ is broken - it is not always the case that the likelihood of $\varphi$  is a function of the likelihood of  $\neg \varphi$.

Lastly we consider a new class of direct aggregators that are not reducible to  ``classical'' JA aggregators.

\paragraph{Most likely prime implicant}
%\MIcomment{not clear if it is resolute?}

One of the oldest and most studied aggregators in ``classical'' JA is the so called {\em premise-based procedure} (PBP) \cite{premise10}.   Some aggregation problems are such that the agenda can be naturally split into two sets: conclusions (or decisions) and premises (or reasons why a decision is taken). For example, the agenda in the example in Table~\ref{tab:co2example} can naturally be split into an agenda of premises $\A_p = \{ p, p\rightarrow q\}$ and an agenda of conclusions $\A_c=\{q\}$. The PBP aggregator works in two steps: first the majority is calculated for each issue in the premise agenda subset. In  the example in Table~\ref{tab:co2example} this would yield the set of premises $\{p, p\rightarrow q\}$. Then the constraint is used to entail the judgments on the issues in the conclusion agenda subset. In  the example in Table~\ref{tab:co2example} this would yield the collective judgment set  $\{p, p\rightarrow q, q\}$. PBP is an aggregator that has many good properties, but it is only applicable to agendas that are split into premises and conclusions. Here we propose a likelihood judgment aggregator that ``operates'' in the same way as PBP but it is applicable to any agenda.

When the agenda is split into premises and conclusions, the problem is such that the judgments on the premises entail each of the judgments on  the conclusions. From a logical perspective, the set of premises is {\em an implicant} of the agenda. Let us formally define this concept generalizing it to any agenda.

\begin{definition}
Given an agenda $\A$ (not explicitly partitioned into premises and conclusions) and constraints $\Ct$ we say that the set $I \subseteq \A^\cup$ is an implicant of $\A$ if $I$ is a consistent (with respect to $\Ct$) set and  either $I \cup \Ct \models \ai$  or $I \cup \Ct \models \neg \ai$,  for every $\ai \in \A^\cup \setminus I$. $I$ is a {\it prime implicant} of $\A$ if $I$ is an implicant and there exists no smaller set $I'$ ($I'\subset I$) that is also an implicant of $\A$.
\end{definition}
%\MIcomment{Ova mozi i kako formalna definicija da odi} 
Consider the agenda and constraints of Example~\ref{ex:running}. This agenda has eight prime implicants, \ie all the consistent three-element subsets of  $\{x,\neg x, e, \neg e, s\vee t, \neg (s\vee t), a, \neg a\}$. %\MIcomment{ovde bez $a, \neg a$, posto toa sleduva od x i e?} 

%The reason  why we do not define implicants on $\A$, which would yield a unique prime implicant in this example, is because of how we use prime implicants to define aggregators. 
%\MIcomment{LAST SENTENCE NOT CLEAR INDEED! P7C2, SEE ALL P7C2}

We can define a class of irresolute likelihood aggregation functions based on agenda prime implicants and a definition of {\em most likely} prime implicant. There are several ways a most likely prime implicant can be defined. We give a few examples.  Let $\mathcal{I}(\A)$ be the set of all prime implicants of $\A$. Then  the  most likely prime implicant of $\A$ is the one with:

\begin{itemize}
\item the highest sum of average likelihoods \\  $ \underset{I \in \mathcal{I}(\A)}{\textrm{argmax}}
\sum_{\ai \in I}  \avg{\ai} $ 
\item the highest minimum average likelihood, \\   
$\underset{I \in \mathcal{I}(\A)}{\textrm{argmax}}
\min_{\ai \in I}  \avg{\ai} $ 
\item the highest number of  majority supported judgments,\\
 $\underset{I\in \mathcal{I}(\A)}{\textrm{argmax}}
 \sum_{\ai \in I} \hat{N}(\Ph,\varphi, c)$, etc.
\end{itemize}
Note that the three definitions given above determine three (possibly) different most likely prime implicants for which we could use different names, but for simplicity we omit that.

 Once a  most likely prime implicant $I^{\ast}$ is determined in one of the above described ways, the collective judgment is a union of $I^{\ast}$ and the elements of $\A^{\cup}$ implied by $I^{\ast}$.

\begin{example} For the profile, agenda and constraints of Example~\ref{ex:running}, the prime implicant that has the highest sum of average likelihoods is $\{s\vee t, x, \neg e\}$ yielding  the collective outcome of $\{ s\vee t, x, \neg e,  h, a\}$.
\end{example}

 If the agenda and constraints are given in DNF (disjunctive normal form), the prime implicants can be found in polynomial time  \cite{STRZEMECKI1992}. To the best of our knowledge, prime implicants have not been used to define aggregation functions in judgment aggregation, with the possible exception of  \cite{JELIA2014} where a distance based function for measuring dissimilarity between two classical judgment sets based on prime implicants   has been defined. %\MIcomment{so ima vrska posledniot muabet so nas? moze toj rezultat da go generalizirame nekako ili?}

 \section{Properties of aggregators}
 \label{sec:properties}
Having generalized the classical judgment aggregation framework, the immediate question to consider is whether the typical impossibility properties  results  also hold for aggregators of probabilistic judgments. To establish this, we need to generalize the definitions of aggregation properties. We also need to see whether there are new interesting desirable properties that need to be considered in the new framework. We begin some of this work here. 
 
 %In this section we define some basic desirable properties for  likelihood profile aggregators. 
 
 We begin by exploring the ``classical'' impossibility   theorem \cite{Dietrich2007}. For this we need to define a resolute likelihood aggregator.  We then define the properties of universal domain, unanimity, rationality and systematicity.
 
 \begin{definition}\label{def:resolute}
 Let $\paf$ be a probabilistic aggregation frame and let $\app$ be the set of all rational likelihood profiles for it. A likelihood (resolute) aggregator $R$ is  a mapping $R: \app \rightarrow \acj$ from the set of rational likelihood profiles to the set of consistent  and complete crisp judgment sets.
 \end{definition}
 In other words, a (crispifying) aggregator is resolute if it assigns only one collective judgment set to each profile.  
 
  {\bf Universal domain}  is the requirement that an aggregator $R$ has to be defined for all the probabilistic rational profiles (and all allowable crispifying vectors where applicable). {\bf Rationality} is the property that $R$ produces only consistent  and complete crisp judgment sets.  These properties are embedded in the Definition~\ref{def:resolute}.  
  
  An aggregator $R$ is {\bf dictatorial} if there is an information source $k\in N$ (a dictator) such that for each likelihood profile $\Ph= (\Jp_1,\ldots, \Jp_n )$, the collective judgment is equal to the  collective judgment on the profile $\Ph_k= (\Jp_k, \ldots,\Jp_k,\ldots, \Jp_k )$, i.e., only the judgment set of the dictator is considered in the aggregation process. Non-dictatorship is the requirement that no information source is a dictator.

  In ``classical'' judgment aggregation, {\bf unanimity} is the property  requiring  that if a judgment is in every judgment set in the profile it has to be in the collective judgment set as well. When aggregating likelihood profiles, unanimity has to be defined with respect to some crispifying coefficient $c$, regardless of whether $R$ is a direct aggregator or not. 
  \begin{definition}[Unanimity]
  Let $c \in [0,1]$. The aggregator $R$ satisfies  $\mathbf{c}${\bf-unanimity} if for every profile of rational probabilistic judgments $\Ph \in \app$, $\Ph= (\Jp_1, \ldots, \Jp_n)$, and every $\ai \in \A^\cup$ it holds that: if $\forall k \in N$: $a^k_{\ai} \geq c$, then $\ai \in R(\Ph)$.
\end{definition}
  
   Lastly we define {\bf systematicity}. Intuitively,  systematicity   is satisfied if every two issues that are judged as equally probable in two different profiles are treated equivalently by the aggregation rule $R$.
 
% \begin{definition}[Systematicity]
 % Given a profile $\Ph\in \app$, let us define $\Ph_i$ to be the column vector of all judgments on $\ai_i\in \A^\cup$ given by the agents in $N$ according to the profile $\Ph$:  $\Ph_i=(\ell(\ai_i) \geq^* a^k_i  \mid k \in N)$.
%The aggregator $R$ satisfies  {\em systematicity}, if for every two profiles $\Ph,\Ph' \in \app$ and every two issues $\ai_i, \ai_j \in \A^\cup$, the following  holds: 
%$\Ph_i = \Ph'_i$ iff [$\varphi_i \in R(\Ph)$ and $\varphi_i \in R(\Ph')$].
% \end{definition}
 
 \begin{definition}[Systematicity]
  Given a profile $\Ph\in \app$, let us define $\Ph_{\ai}$ to be the {\it projection} of $\Ph$ on the issue $\ai$: $\Ph_{\ai}=(\ell(\ai) \geq^* a^k_{\ai}  \mid k \in N)$.
The aggregator $R$ satisfies  {\em systematicity}, if for every two profiles $\Ph,\Ph' \in \app$ and   issues $\ai, \psi \in \A^\cup$, the following  holds: 
$\Ph_{\ai} = \Ph'_{\psi}$ implies [$\ai \in R(\Ph)$ iff $\psi \in R(\Ph')$].
 \end{definition}
 
 The following theorem can easily be proved following the proof method of  Theorem 3.7. in \cite{GrossiP:2014}.
 
 \begin{theorem}
 Consider a frame $\paf$. Let $\app$ be the set of all rational likelihood profiles that can be defined for the given frame. The aggregation function $R$ satisfies unanimity, rationality and systematicity if and only if $R$ is a dictatorial aggregation function.
 \end{theorem}
 
For other desirable properties that could be applied in our framework we  can look to the literature of   probabilistic opinion pooling \cite{MartiniS2017,DietrichList2017}  for inspiration. 

{The systematicity requirement corresponds to a property called \textbf{Strong setwise function property (SSFP)} in opinion pooling \cite{MartiniS2017}.  This property requires that the group probability of an event $A$ depends only on the individually assigned probabilities of $A$. %if all the individually assigned probabilities of an event are equal, then the pooled probability of this event is equal to the same value. 
It was shown  that SSFP gives rise to impossibility results in opinion pooling \cite{MartiniS2017}. A weaker property, that is implied by  SSFP,  is the \textbf{Zero Preservation Property (ZPP)}. ZPP is satisfied when for profiles where  all of the individually assigned probabilities of an issue are zero,  the collective probability of this issue is also zero.  The ZPP property is related to the unanimity property in ``classical'' JA. Here we can define ZPP as {\bf $1$-unanimity}, namely $\mathbf{c}${\bf-unanimity} where $c = 1$.}  

 %\begin{definition}[Zero Preservation Property] Let $\F: \app  \to %\mathcal{P}(\acj)$. We say that $F$ satisfies ZPP when for every $\ai %\in\A^{\cup}$ s.t.   $\ell(\varphi)= 1$ ($\ell(\neg\varphi)=0$, resp.) is in %every individual probabilistic judgment set in $\Ph$, it follows that $\varphi %\in \Js$ for every $\Js \in F(\Ph)$.
 %\end{definition}
 
Intuitively, 1-unanimity is desirable: whenever every source is sure that an issue is true (or false), a 1-unanimity satisfying aggregator will capture that certainty.   However,  unanimity on $\ell(\varphi)= 1$ for $\varphi$ does not mean that a rational judgment set $\Js$ such that  $\varphi \in \Js$ exists! Recall that we have constraints $\Ct$  imposed by the agenda setter on the outcome but not on the judgment sets of the profile. This means that 1-unanimity can only be satisfied in a specific aggregation frame if the set of constraints $\Ct$ of the frame allows it. Let us give an example. 

\begin{example}
Let $\paf$ be a probabilistic aggregation frame such that $\varphi_1, \varphi_2 \in \A$ and $\{ \varphi_1 \rightarrow \varphi_2 \} \in \Ct$.  Let $\Ph$ be a consistent probabilistic profile in this frame such that $\ell(\varphi_1)= 1$ and  $\ell(\varphi_2)= 0$ are in every individual judgment set in $\Ph$.
Then 1-unanimity requires that $\{\varphi_1, \neg \varphi_2 \}\subseteq \Js$  for every $\Js \in R(\Ph)$ but none such $\Js$ is  in $\acj$.
\end{example}

 As seen in the above example, whether an aggregation function can satisfy 1-unanimity is a property of the aggregation frame. We can define the ZPP property for likelihood aggregators $R$ (the definition also extends to irresolute direct and indirect aggregators as well) as 1-unanimity when the constraints allow it. 

\begin{definition}[Zero preservation property]
Let $\paf$ be a probabilistic aggregation frame. Given a profile $\Ph \in \app$, we define the set $Z(\Ph)$ as
$$Z(\Ph) = \{\ai \in \A^\cup:  a^k_{\ai}=1,\; \forall k \in N\}.$$ 
Namely, the set $Z(\Ph)$ contains all the agenda issues that have been unanimously awarded likelihood 1 in $\Ph$. We say that an aggregator $R$ satisfies the zero preservation property if for all $\Ph \in \app$ it holds that  if $Z(\Ph)$ is consistent with $\Ct$, then $Z(\Ph)\subseteq R(\Ph)$.  
\end{definition}

We consider one more  intuitive property of probabilistic aggregators, that of \textbf{convexity} \cite{MartiniS2017}. Convexity states that the collectively assigned  (minimal) probability on an issue should be a value no smaller than the smallest and no higher than the highest individually assigned probability on that issue. For the direct aggregators this property is not applicable, but it is so  for the crispifying ones.
 \begin{definition}[Convexity] Let $\Jp$ be one of the probabilistic judgment sets assigned to profile $\Ph$ by a function $\hat{F}$  (before crispification).  For a given $\ai \in\A^{\cup}$, let 
 $a^{\max}_{\varphi}=\max_{k}a^k_{\varphi}$ and  $a^{\min}_{\varphi}=\min_{k}a^k_{\varphi}$. We say that $\hat{F}$ satisfies  convexity when for all collective $\Jp$ and every $\ai \in \A^{\cup}$, if  $\ell(\varphi)\geq^* a_{\varphi}\in \Jp$ then $a^{\min}_{\varphi}\leq a_{\varphi} \leq a^{\max}_{\varphi}$.
 \end{definition}
 
%\MS{We first show how irresolute likelihood aggregators can be compared with ``classical" irresolute aggregators}

 It can directly be observed that $\hat{f}_{\qu}(\Ph)$  satisfies  convexity  and that $\Fh_{\qu}(\Ph, \qv)$  satisfies ZPP and universal domain. However $\Fh_{\qu}(\Ph, \qv)$  does not satisfy rationality  because the sets in $\Fh_{\qu}(\Ph, \qv)$  are not always   rational sets of judgments.  

The $\Fq$ aggregator satisfies  universal domain and non-dictatorship by design. It is  clear that  $\hat{F}_S$ also satisfies ZPP. 
 
 The distance-based direct aggregators satisfy universal domain and non-dictatorship by design, however it is safe to conjecture 
 %\MIcomment{WHY A CONJECTURE AND NOT A LEMMA? P7C1L-6} \MScomment{Oti nemavme vreme da go proverime i precizno da go formulirame...a ni mesto} 
 that they will not satisfy ZPP for the same reason that distance-based classical aggregators do not satisfy unanimity \cite{ADT09} -- for an agenda with  sufficiently many issues, a judgment set that does not contain the unanimously likely judgment might end up being  ``closer'' to the profile. 
 
 With the most likely prime implicant class of aggregators,  universal domain and rationality will be satisfied by design. However, ZPP will not be satisfied - when the unanimously supported issue is not in the prime implicant its inclusion in the collective judgment set will not be guaranteed.
 
 %\MIcomment{ova kompletna diskusija e, t.e. neznaeme poveke da kazeme vo vrska so properties na definiranite agregatori?}

\section{Related work}\label{sec:related}

As mentioned in the introduction, the area of probabilistic opinion pooling is concerned with aggregating probability functions into a single one.  As opposed to {standard} probabilistic opinion pooling, our logic-based approach:
\begin{enumerate}
    \item  allows for an arbitrary agenda, namely instead of taking the entire $\sigma$-algebra,   the agenda can be limited to the important issues of consideration in the actual context;
    \item We do not limit ourselves to expressing point probabilities over the issues (but we do include that option as well); 
    \item The result of the opinion aggregation is a set of propositional statements, hence a final decision, and not a probabilistic consensus.
\end{enumerate}
  Dietrich and List \cite{DietrichList2017} generalize  opinion pooling to general agendas and examine properties and impossibility results. However, their work does not define any particular aggregators and also 2) and 3) are not the case there.

The problem of transforming degrees of belief into binary beliefs is known as belief-binarization. Dietrich and List \cite{DietrichList2018} study how a profile of Boolean judgments, that has been transformed into a vector of beliefs (for example a profile from Table~\ref{tab:co2example} becomes the vector $(\frac{2}{3}, \frac{2}{3}, \frac{1}{3})$) can be  ``binarized'' into a consistent set of Boolean judgments. In \cite{DietrichList2018}, however, only binary profiles are aggregated. 

There are several approaches towards aggregating imprecise probabilities (IP), like aggregation of probability intervals in \cite{moral1998}, subjective opinion fusion in \cite{josang2016}, etc. More recently \cite{stewart2018} extended pooling properties to IPs using convex functions.  Moreover, they go further and aggregate precise probabilities into imprecise (the convex hull of the input probabilities as a proof of concept) arguing that IP models are better suited as models of rational consensus.  
Allowing for inequalities in the likelihood judgments, we allow for modeling IP in the individual judgments. However, unlike \cite{stewart2018} we require the collective judgment to be crisp {since our goal is to define specific aggregators that support the decision making in various contexts. }

Dietrich and List \cite{Dietrich2007} generalize classical JA assuming formulas from a {\em general logic} and prove impossibility theorems. They show that the model is applicable to, for example,  propositional, modal, and many-valued logic. The model is not directly applicable to the likelihood logic we use here, since it assumes that the agenda issues are formulas in the particular logic. Since it does not make sense to choose a finite agenda of likelihood formulas, we express the issues in propositional logic and use likelihood formulas for the judgments. The latter makes our framework fit better in the general theory of aggregation of {\em propositional attitudes} in \cite{DietrichList2010}, that integrates probabilistic opinion pooling and judgment aggregation. In this theory, profiles consist of {\em attitude functions} (which can be probability functions, truth-value functions, etc.) defined over finite subset of a $\sigma$-algebra (an agenda).
%\cite{Herzberg2013}  goes further and takes many-valued algebras as sets of values.  
We believe that defining notions on the level of a syntax has certain advantages, explicitly defining the concept of rationality being one of them.  
%In the theory of propositional attitudes,  rationality  is defined in a rather abstract way: a judgment is {\em rational} if the corresponding attitude function is extendable to a well-defined valuation function on the entire $\sigma$-algebra.}

It is not always possible to have complete information, sometimes some sources will not be able to provide information on all of the issues. Although impossibility results involving abstentions have been shown \cite{Dokow2010}, designing functions to aggregate the so called {\em incomplete judgments} is not given a lot of attention in the JA literature  \cite{TerzopoulouEtAlCOMSOC2018,thesis}.  By showing how crisp profiles can be represented as likelihood ones and designing likelihood judgment aggregators, we enable probabilistic judgment aggregation to also be used for aggregating crisp incomplete judgment sets in a straightforward way. 

Interpreting the likelihood operator with a possibility measure leads to the formula $\ell(\varphi)\geq a$ being equivalent to the formula $(\varphi, a)$ as well as the (uniform) crispifying being equivalent to $\alpha$-cut in possibility theory.
There are various methods of information fusion considered in this theory \cite{DuboisPrade2001}. However, they focus on 
%(as discussed for example in Possibility Theory in Information Fusion by D.Dubois and H.Prade) 
merging  information about the true state of a variable or a proposition
and take a set theoretic approach to defining the merging functions 
%and do not consider an agenda of issues that could possibly influence the (choice of) definition for information fusion.
while we, on the other hand, follow the tradition of judgment aggregation and social-choice theory and take an agenda of (logically related) issues as a starting point. This means that both the (choice of) definitions of the aggregators, and the choice of crispifying coefficients depend on the agenda. 
%(if the agenda contains premises and conclusions, for example).
Moreover, our goal is not just to merge the imprecise information coming from the different sources, but to make a decision about the true state of the agenda issues.

Probabilistic belief merging is considered in \cite{PotykaATS16,POTYKA2017}. In belief merging sets of formulas, possibly likelihood formulas, are called knowledge bases. Knowledge bases from several sources, that can be mutually inconsistent,  are merged to obtain a consistent knowledge base. The difference between belief merging and judgment aggregation has been analyzed in \cite{EveraereKM15}. Essentially, in belief merging the knowledge bases do not share the same agenda,  which entails different properties to be desired for the merging operators as compared to the desired properties for judgment aggregation functions. 
%for which the sources provide us with (possibly imprecise) probabilistic information and we want to make a strict decision about the truth of the issues in the agenda. 
%The proposed aggregators are defined based on a given agenda, and the structure of the agenda along with the context plays a significant role in the choice of the aggregator.

%The alpha cut in possibility theory is equivalent to a uniform crispifying with a single coefficient alpha in our platform (under the assumption that we interpret likelihood as possibility). However, in the general case we consider crispifying with a (not completely arbitrary) vector of coefficients, one for each issue and its negation, which once again points to the role of the agenda as a starting point in our platform.

\section{Conclusions and future work} \label{sec:conclusion}

%\MIcomment{treba da se vidat questions for rebuttal od reviewer 2 i 3.

We consider as the main contribution of our paper the definition of various functions for aggregating likelihood judgments on logically related issues. Furthermore, we show how these aggregators relate to classical judgment aggregation function, and in turn, through the results shown in \cite{ADT2013} and \cite{EndrissAAMAS2018}, how likelihood judgment aggregation relates to voting methods. 
We also define desirable properties for the aggregation functions and show that the classical impossibility results hold here as well. 

Some more consideration needs to be given to further distinguish the likelihood profile aggregators. From the examples we can observe that very different outcomes are produced for the same profile by different aggregators. The minimal set of properties we discuss is not sufficient to allow a user to choose which is the best aggregator for a given probabilistic frame. In light of new properties, particularly the direct aggregators need to be carefully studied. 

More properties from opinion pooling can be considered and ``translated'' into our extended JA framework. An interesting candidate is the so called \textbf{Independence Preservation (IP)} property which intuitively requires that if two issues are probabilistically independent according to all information sources, then this independence should be preserved in the collectively assigned probabilities for the two issues. {This has been explored by Wagner in \cite{Wagner84} for the case of aggregating point probabilities over an agenda of mutually exclusive events; \cite{stewart2018} explores the imprecise probabilities case. Note that to represent probabilistic independencies in the judgments, we need to either extend the logic of likelihood with polynomial likelihood formulas or include likelihood independence formulas as defined in \cite{IvanovskaGiese10a} directly in the language. Then we could define IP properties alike \cite{Wagner84} and \cite{stewart2018} and see what are the consequencies of their impossibility results in our platform.
%For the case of aggregating point probabilities over an agenda of mutually exclusive events,  Wagner in \cite{Wagner84} shows that the only aggregation functions satisfying {\em irrelevance of alternatives} and IP are either dictatorial or functions that assign probability 1 to one of the elements in the agenda and 0 to all the rest, i.e. functions assigning a crisp judgment set to all probabilistic profiles\footnote{This is the result for the case when probabilistic independence between two events $A$ and $B$ is defined as $p(A\wedge B)=p(A)p(B)$ and $p(A)$, $p(B)$ are allowed to be zeros.} It would be interesting to see the implications of this result in a general agenda with logically related issues. 
We notice also that the IP property} is reminiscent of the agenda separability property studied in \cite{AAAI}. One direction of future work is to establish this intuitive connection and explore other such connections between probability aggregation properties and JA properties.

We believe that with this work we have made several contributions to the classical JA theory as well:  We have significantly extended the classical binary judgment aggregation framework, opening up this social choice method for applications in new AI domains, particularly involving the aggregation of uncertain judgments; our framework allows for not only uncertainties but also abstentions to be modelled using  $\ell(\varphi) \geq 0$, which is a neglected feature in judgment aggregation frameworks overall; furthermore, we generalize the assumption that the same relations between issues should hold for both the information sources and the aggregated result. This ``double constraint'' framework is actually very intuitive \cite{EndrissAAMAS2018}.

Having a probabilistic framework also opens  possibilities to study the {\em truth-tracking } properties of judgment aggregators, namely how good is a function in aggregating profiles into  the most likely judgments. This area of judgment aggregation is still relatively little explored \cite{Bozbay2019}. We intend to explore truth-tracking in future work. 

We believe that there is a possibility for applying our work to prediction markets \cite{WolfersZitzewitz2004}, specifically in extending the agenda of predictions (which is typically consisting of states of a random variable) to a set of logically related statements. Prediction markets \cite{BarbuLay2012} are forums for trading contracts for outcomes of future events. Each market participant possesses certain information about the event in question, and conveys this information to the market by the way she trades contracts. The contract price is a result of aggregation of the information possessed by all the participants, hence is an estimator of the probability of the event in question. We believe we could seek for inspiration in defining new aggregators under our platform by studying  the methods of information fusion that various prediction markets apply.

%\clearpage

%\nocite{*}
\bibliographystyle{eptcs}
\bibliography{Biblio.bib}
\end{document}